\documentclass[runningheads]{llncs}
\usepackage[T1]{fontenc}
\usepackage[utf8]{inputenc}
\usepackage{amsmath,amssymb}
\usepackage{mathtools}
\usepackage{stmaryrd}
\usepackage{booktabs}
\usepackage{array}
\usepackage{multirow}
\usepackage{enumitem}
\usepackage{xcolor}
\usepackage{listings}
\usepackage{graphicx}
\usepackage{hyperref}
\usepackage{doi}
\usepackage[capitalise,noabbrev]{cleveref}
\usepackage{tikz}
\usepackage{tabularx}
\usetikzlibrary{arrows.meta, positioning, fit, backgrounds}
\newcommand{\Policy}{\mathcal{P}}
\newcommand{\PolicyOffer}{\Policy_{\text{offer}}}
\newcommand{\PolicyReq}{\Policy_{\text{req}}}

\newcommand{\sem}[1]{\llbracket #1 \rrbracket}
\newcommand{\verdict}{\operatorname{verdict}}

\newcommand{\Compatible}{\textsc{Compatible}}
\newcommand{\Conflict}{\textsc{Conflict}}
\newcommand{\Unknown}{\textsc{Unknown}}

\newcommand{\odrl}[1]{\textsf{#1}}

\newcommand{\pr}[1]{\pi_{#1}}
\newcommand{\Decomp}[1]{\mathit{Decomp}(#1)}
\newcommand{\axV}[1]{V_{#1}}
\newcommand{\boxV}{V_{\mathrm{box}}}
\newcommand{\boxS}{S_{\mathrm{box}}}
\newcommand{\R}{\mathbb{R}}
\newcommand{\Rpos}{\mathbb{R}_{\geq 0}}
\newcommand{\Verdicts}{\mathbb{V}}   
\spnewtheorem{assumption}[theorem]{Assumption}{\bfseries}{\itshape}
\crefname{definition}{Definition}{Definitions}
\Crefname{definition}{Definition}{Definitions}
\crefname{theorem}{Theorem}{Theorems}
\Crefname{theorem}{Theorem}{Theorems}
\crefname{lemma}{Lemma}{Lemmas}
\Crefname{lemma}{Lemma}{Lemmas}
\crefname{proposition}{Proposition}{Propositions}
\Crefname{proposition}{Proposition}{Propositions}
\crefname{corollary}{Corollary}{Corollaries}
\Crefname{corollary}{Corollary}{Corollaries}
\crefname{remark}{Remark}{Remarks}
\Crefname{remark}{Remark}{Remarks}
\crefname{example}{Example}{Examples}
\Crefname{example}{Example}{Examples}
\crefname{assumption}{Assumption}{Assumptions}
\Crefname{assumption}{Assumption}{Assumptions}
\begin{document}
\title{Axis-Aligned Semantics for ODRL: Resolving Dimensional Ambiguity in Policy Constraints}
\titlerunning{Axis-Decomposed ODRL Semantics}
\author{Daham M. Mustafa\inst{1,2} \and
Diego Collarana\inst{2} \and
Sabrina Kirrane\inst{4} \and
Christoph Lange\inst{1,2} \and
Christoph Quix\inst{1,2} \and
Rafiqul Haque\inst{3} \and
Yixin Peng\inst{1} \and
Stefan Decker\inst{1,2}}
\authorrunning{D. M. Mustafa et al.}
\institute{RWTH Aachen University, Aachen, Germany\\
\email{daham.mohammed.mustafa@fit.fraunhofer.de}
\and
Fraunhofer FIT, Sankt Augustin, Germany
\and
University of Galway, Galway, Ireland
\and
Vienna University of Economics and Business (WU), Vienna, Austria}
\maketitle
\begin{abstract}
The Open Digital Rights Language (ODRL) represents policy
constraints as triples of a left operand, an operator, and a
value. Several spatial operands, however, range over
multi-axis domains such as width, height, and depth, while the
constraint syntax provides no explicit axis identity. As a
result, policy engines cannot determine whether multiple
constraints apply to the same axis or different ones, making
conflict detection unsound or incomplete. We resolve this ambiguity by axis decomposition, replacing multi-axis operands with axis-specific scalar operands over
totally ordered domains. Each constraint then denotes an
interval per axis and each policy an axis-aligned box,
reducing conflict detection to box comparison. We define a
three-valued semantics (\Conflict, \Compatible, \Unknown),
prove the decomposition sound and backward compatible with
ODRL, instantiate it as ODRL Axis-Aligned Profile (OAAP)\footnote{\url{https://github.com/Daham-Mustaf/ODRL-Axis-Aligned-Profile-OAAP}}, and validate it on a benchmark of 256 ODRL policy problems, each expressed in Turtle and compiled to first-order (TPTP) and SMT-LIB form, using Vampire, E,
Z3, and cvc5.
\end{abstract}

\section{Introduction}
\label{sec:introduction}
Data spaces increasingly exchange multi-axis digital assets such as
high-resolution imagery, geospatial data, and 3D cultural heritage scans. These assets are naturally described along two or three independent ordered axes, $(x,y)$ or $(x,y,z)$: raster images by width and height, 3D scans by width, height, and depth, and geo-referenced assets by longitude, latitude, and altitude. Policies must therefore constrain each axis independently.

Such constraints are typically expressed using policy languages such as the Open Digital Rights Language ODRL~\cite{ODRL22}, the W3C standard for access and usage policies, which is recommended by the Data Spaces Support Center for governing data transfers between participating organizations~\cite{dssc2024blueprint}, including in the Culture Dataspace (Datenraum Kultur)~\cite{ToubekisDecker2025}, a data space connecting libraries, archives, and museums in Germany. Its assets are stored in dimensional formats such as master TIFF images~\cite{loctifftags}, with fixed pixel width and height, and 3D scans with width, height, and depth~\cite{mdz-3d}.

ODRL represents each constraint as a triple consisting of a left operand, an operator, and a value, without explicit axis identity. Left operands range over multi-axis domains, namely \odrl{absoluteSize}, \odrl{relativeSize},
\odrl{absoluteSpatialPosition}, \odrl{relativeSpatialPosition}, and \odrl{spatialCoordinates}. A multi-axis operand is therefore treated as a single undifferentiated attribute, so constraints over the same operand cannot be distinguished by the axis they refer to. For example, the constraint (\odrl{absoluteSize}, \odrl{lteq}, 1920) specifies a numeric bound without indicating whether it applies to width, height, or depth. This ambiguity becomes critical when multiple constraints coexist, e.g., a policy combining:
(\odrl{absoluteSize}, \odrl{lteq}, 1920), \odrl{and}
(\odrl{absoluteSize}, \odrl{lteq}, 1080), \odrl{and}
(\odrl{absoluteSize}, \odrl{lteq}, 50),
which cannot be interpreted unambiguously as applying to a single axis or multiple axes.

ODRL has been formalised via first-order logic~\cite{pucella2006}, rule- and trace-based semantics~\cite{steyskal2015,bonatti2025towards}, query containment~\cite{salas2025odrl}, constraint normalisation~\cite{salas2026normalisation}, and the W3C Community Group's draft formal semantics~\cite{odrlfs2026}. All treat each left operand as a single ordered domain, collapsing multi-axis operands into undifferentiated attributes. A policy engine therefore cannot tell whether constraints share an axis or target different ones, so conflict detection is unsound or incomplete.

ODRL provides a \emph{profile extension} mechanism for adding left operands without changing constraint syntax~\cite{ODRL22}. We use it to give a general axis-decomposition method: any operand whose value space is a product of totally ordered axis domains is replaced by a scalar operand for each axis, so every constraint names the axis it bounds. On the decomposed operands we develop a compositional denotational semantics interpreting each constraint as a per-axis interval and each constraint set as an axis-aligned box, with a two-layer procedure whose per-axis verdicts compose into a box verdict over {$\Conflict$, $\Compatible$, $\Unknown$} where \Unknown\ marks a missing axis. Operands left undecomposed are unchanged, so the resulting profiles are non-interfering extensions. We instantiate it for the five spatial operands as the ODRL Axis-Aligned Profile (OAAP), an OWL~2 profile with 15 axis operands and SHACL shapes, evaluated on 256 TPTP and SMT-LIB problems with Vampire, E, Z3, and cvc5.

The rest of the paper is organised as follows. Section~\ref{sec:preliminaries} presents preliminaries and motivating example. 
\cref{sec:problem} defines the axis profile and its interval semantics; \cref{sec:conflict} builds the per-axis and box verdicts, their \odrl{or}/\odrl{xone} composition, and the runtime bridge; \cref{sec:implementation} gives OAAP and the evaluation; \cref{sec:related,sec:conclusion} cover related work and conclusions.
\vspace{-10pt}
\section{Preliminaries and Motivating Example}
\label{sec:preliminaries}
We work at the constraint layer of ODRL. Following ODRL
and a rule's constraint set combines constraints via the logical operators \odrl{and},\odrl{or}, odrl{xone}, and \odrl{andSequence}~\cite{odrl-vocab}. We treat the first three, which are order-independent; \odrl{andSequence} fixes an order on its constraints and can leave the outcome unresolvable, so it lies outside the interval semantics developed here. A rule is active when its constraint set is
satisfied and inactive otherwise~\cite{odrlfs2026}. A left operand fixes its admissible right-operand values and the
operators on them, and we classify operands by the source of this
value space. Operands whose domain comes from an external resource
referenced via $\odrl{rightOperandReference}$ (e.g.\ $\odrl{purpose}$,
$\odrl{spatial}$) are \emph{externally-grounded}. They require
open-world reasoning over that resource and are out of scope of this paper. The
remaining operands are \emph{self-contained}, with a value set fixed
by ODRL specification that is real-valued and totally ordered. A
self-contained operand is \emph{scalar} when it ranges over a single
domain and \emph{dimensional} when its value set is a product of two or three such domains, one per axis. Table~\ref{tab:dimentional}
lists ODRL's dimensional left operands. For example,
$\odrl{absoluteSize}$ ranges over a product of width, height, and
depth, each in $\Rpos$, defined by ODRL vocabulary as
``measure(s) of one or two axes for 2D-objects or measure(s) of one
to three axes for 3D-objects''~\cite{odrl-vocab}.
\begin{table}[!t]
\caption{Dimensional ODRL left operands. Each value space is a
product of totally ordered real-valued domains, one per axis.
$L_i$ denotes the extent of axis $i$.}
\label{tab:dimentional}
\centering
\small
\begin{tabular}{@{}lll@{}}
\toprule
Operand class & Value space & Operands \\
\midrule
Asset size     & $\prod_{i=1}^{n}(0,\infty)$, $n \in \{2,3\}$
               & \odrl{absoluteSize}, \odrl{relativeSize} \\
Asset position & $\prod_{i=1}^{n}[0, L_i]$, $n \in \{2,3\}$
               & \odrl{absoluteSpatialPosition},
                 \odrl{relativeSpatialPosition} \\
Coordinates    & $\R^{n}$, $n \in \{2,3\}$
               & \odrl{spatialCoordinates} \\
\bottomrule
\end{tabular}
\end{table}
\subsection{Motivating Example}
\label{sec:motivation}
In the Culture Dataspace, the \emph{provider}, the Bavarian State
Library (BSB)~\cite{bavarikon}, publishes a 3D scan of a historical
manuscript (\texttt{drk:bsb-clm-14000-3d}) from the Munich
Digitization Centre~\cite{mdz-3d} under an offer policy
$\PolicyOffer$ bounding the rendered width, height, and depth. The
\emph{consumer}, the French National Library (BnF), submits a
request policy $\PolicyReq$ constraining width and height and
leaving depth unconstrained. Both express their bounds through the
single operand $\odrl{absoluteSize}$, with the offer stating three
thresholds $1920$, $1080$, and $50$ and the request stating two values
$2400$ and $800$ (Table~\ref{tab:example-constraints}). Reading
these five scalar comparisons, a policy engine cannot tell which
targets which axis, so it can declare neither conflict nor
compatibility. OAAP replaces $\odrl{absoluteSize}$ with three axis-specific
operands, \texttt{oaap:absoluteSizeWidth},
\texttt{oaap:absoluteSizeHeight}, and
\texttt{oaap:absoluteSizeDepth}, each over the range $(0, \infty)$, so every
constraint names its axis and the comparison is well-defined per
axis.
\begin{example}[BSB--BnF manuscript display]\label{ex:bsb}
Under the OAAP profile, the offer constrains all three axes and the
request constrains width and height only. The width pair compares
$\odrl{lteq}~1920$ against $\odrl{eq}~2400$ on the same axis, the value sets are disjoint, and the pair is a \Conflict. The height
pair compares $\odrl{lteq}~1080$ against $\odrl{eq}~800$, where
$800$ lies in $(0, 1080]$, so the pair is \Compatible. The request imposes no constraint on depth. A silent axis could later be completed into either compatibility or conflict, so the offer's depth bound has no decidable counterpart and the depth verdict is \Unknown.
\end{example}
\noindent
\begin{table}[!t]
\footnotesize
\centering
\caption{
Per-axis constraints and verdicts for
Example~\ref{ex:bsb}. The offer ($c_i$, from provider BSB)
and request ($c'_i$, from consumer BnF) are decomposed under
OAAP from \texttt{odrl:absoluteSize} into
\texttt{odax:absoluteSize\{Width,Height,Depth\}} (abbreviated
\texttt{:Width}, \texttt{:Height}, \texttt{:Depth}).
Verdicts follow \cref{def:axis-verdict,def:box-verdict}.}
\label{tab:example-constraints}
\renewcommand{\arraystretch}{1.2}
\small
\setlength{\tabcolsep}{4pt}
\begin{tabularx}{\linewidth}{@{}l l c c l X@{}}
\toprule
\textbf{ID} & \textbf{OAAP operand}
  & \textbf{Offer} & \textbf{Request}
  & \textbf{Verdict} & \textbf{Reason} \\
\midrule
$c_1, c'_1$ & \texttt{oaap:Width}
  & $\odrl{lteq}\ 1920$ & $\odrl{eq}\ 2400$
  & \Conflict   & $(0,1920]\cap\{2400\}=\emptyset$ \\
$c_2, c'_2$ & \texttt{oaap:Height}
  & $\odrl{lteq}\ 1080$ & $\odrl{eq}\ 800$
  & \Compatible & $800\in(0,1080]$ \\
$c_3$       & \texttt{oaap:Depth}
  & $\odrl{lteq}\ 50$   & ---
  & \Unknown    &  No depth bound \\
\bottomrule
\end{tabularx}
\end{table}
\vspace{-10pt}
\section{Axis Decomposition and Interval Semantics}
\label{sec:problem}
We define a class of ODRL left operands as \emph{multi-axis}, operands whose value spaces are finite Cartesian products of independently and \emph{linearly} ordered scalars.

\begin{definition}[Multi-axis operand]\label{def:multiaxis}
A \emph{multi-axis} left operand $\ell$ has value space
$\mathcal{V}_\ell = \prod_{k=1}^{n} D_k$ with arity $n \in \{2,3\}$, where each
\emph{axis domain} $D_k$ is a totally ordered subset of $\mathbb{R}$ and the
projection $\pi_k \colon \mathcal{V}_\ell \to D_k$ extracts the $k$-th axis
component. The \emph{axis profile} of $\ell$ is
$\langle n,(D_k)_{k\le n},(\pi_k)_{k\le n}\rangle$.
\end{definition}
\noindent
For example, $\odrl{absoluteSize}$ is multi-axis with value space
$\mathcal{V}_\ell = (0,\infty)^3$ and projections
$\pr{1},\pr{2},\pr{3}$ onto the width, height, and depth axes,
each over the domain $(0,\infty)$. The axis domains may differ;
$\odrl{spatialCoordinates}$ projects onto longitude, latitude,
and altitude over the distinct domains $[-180,180]$, $[-90,90]$,
and $\R$, which is why \cref{def:multiaxis} indexes the domain
$D_k$ by axis. Of ODRL's twelve operators~\cite{odrl-vocab},
five denote a contiguous interval of a totally ordered domain
for every right-operand value.
\begin{definition}[Interval Operator]\label{def:int-op}
An operator $\bowtie$ is an \emph{interval operator} on a totally ordered domain
$D_k$ if for every right operand value $v_k \in D_k$ the set
$\{x \in D_k \mid x \bowtie v_k\}$ is a contiguous interval of $D_k$. The interval
operators are $\mathcal{O}_I = \{\odrl{lt}, \odrl{lteq}, \odrl{gt}, \odrl{gteq},
\odrl{eq}\}$.
\end{definition}
\noindent
The other operators are not interval operators. $\odrl{neq}$ induces
the non-contiguous set $D_k \setminus \{v_k\}$, and the six
set-membership operators ($\odrl{isA}$, $\odrl{hasPart}$,
$\odrl{isPartOf}$, $\odrl{isAllOf}$, $\odrl{isAnyOf}$,
$\odrl{isNoneOf}$) do not act on an ordered domain. Applying an
interval operator to a multi-axis operand therefore requires
decomposing the product value space axis by axis.
\begin{definition}[Axis decomposition]\label{def:profile}
The \emph{axis decomposition} of a multi-axis left operand $\ell$ introduces, for each
axis $k$, a distinct scalar left operand $\ell_{a_k}$ associated with the
projection $\pi_k$ and ranging over the axis domain $D_k$ of \cref{def:multiaxis}.
We write $\Decomp{\ell} = \{\ell_{a_1}, \ldots, \ell_{a_n}\}$.
\end{definition}

\noindent
For a multi-axis left operand such as $\odrl{absoluteSize}$, the
decomposition gives each axis its own scalar operand, so that a
constraint can name the axis it bounds. For absoluteSize, the
three operands are $\texttt{oaap:absoluteSizeWidth}$,
$\texttt{oaap:absoluteSizeHeight}$, and $\texttt{oaap:absoluteSizeDepth}$, each ranging over the positive reals; a size is always greater than zero (a zero measure denotes no asset), with no fixed maximum.
\begin{definition}[Well-formed constraint]\label{def:well-formed}
A constraint $(\ell_{a_k}, \bowtie, v_k)$ is \emph{well-formed} if $v_k \in D_k$
and the set $\{x \in D_k \mid x \bowtie v_k\}$ it denotes is non-empty.
\end{definition}
\noindent
This rules out unsatisfiable constraints, whose denotation is
empty. For example, on the relative-size domain $[0,100]$ the
constraint $(\texttt{oaap:relativeSizeWidth}, \odrl{lt}, 0)$
denotes the empty set, since no $x \in [0,100]$ satisfies
$x < 0$, and is therefore ill-formed.
\vspace{-10pt}
\subsection{Interval Denotation}
\label{sec:interval-denotation}
A constraint on a single axis denotes a subset of that
axis's domain, the constraint's \emph{denotation}. Under an
interval operator this subset is an interval, and two
constraints on the same axis are compared by intersecting
their denotations.
\begin{definition}[Interval operator denotation]
\label{def:interval-denotation}
For a well-formed constraint $c = (\ell_{a_k}, \bowtie, v_k)$ with
$\bowtie \in \mathcal{O}_I$, its \emph{denotation} is the set of
axis values satisfying it,
$\sem{c} = \{x \in D_k \mid x \bowtie v_k\}$, an interval of
$D_k$:
\begin{align*}
  \sem{\ell_{a_k}~\odrl{eq}~v_k}   &= \{v_k\}
    && \text{exactly $v_k$}\\
  \sem{\ell_{a_k}~\odrl{lteq}~v_k} &= [\inf D_k,\, v_k]
    && \text{at most $v_k$}\\
  \sem{\ell_{a_k}~\odrl{gteq}~v_k} &= [v_k,\, \sup D_k]
    && \text{at least $v_k$}\\
  \sem{\ell_{a_k}~\odrl{lt}~v_k}   &= [\inf D_k,\, v_k)
    && \text{strictly below $v_k$}\\
  \sem{\ell_{a_k}~\odrl{gt}~v_k}   &= (v_k,\, \sup D_k]
    && \text{strictly above $v_k$}
\end{align*}
Here $\inf D_k$ and $\sup D_k$ are the lower and upper ends of
the domain; a square bracket includes the end, a parenthesis
excludes it.
\end{definition}
\noindent
Applied to \cref{ex:bsb} on $D_k = (0,\infty)$, the offer's
$\odrl{lteq}$ bounds $c_1, c_2, c_3$ on width, height, and depth
denote $(0, 1920]$, $(0, 1080]$, and $(0, 50]$, each open at $0$
since $0 \notin D_k$. The request's $\odrl{eq}$ constraints
$c'_1, c'_2$ denote the singletons $\{2400\}$ and $\{800\}$.
\begin{lemma}[Well-formed denotation]
\label{lem:totality}
For every well-formed constraint $c$, $\sem{c}$ is a non-empty
interval of $D_k$.
\end{lemma}
\noindent
\begin{proof}
By \cref{def:well-formed}, $\sem{c} = \{x \in D_k \mid x \bowtie v_k\}$
is non-empty, and by \cref{def:int-op} it is a contiguous interval of
$D_k$ for $\bowtie \in \mathcal{O}_I$.\qed
\end{proof}
\noindent
A policy may place several well-formed constraints on the
same axis, e.g., both an upper and a lower bound on width.
The constraint set's denotation on that axis is the
intersection of the per-constraint denotations.
\begin{lemma}[Axis normalisation]
\label{lem:normalisation}
Let $C$ be a finite set of well-formed constraints over
$\Decomp{\ell}$, conjoined under $\odrl{and}$, and $a_k$ an axis. The \emph{normalised
interval of $C$ on $a_k$}, written $I_k(C)$, is
\[
  I_k(C) = \bigcap\{\sem{c} \mid c \in C,
  \; c \text{ targets } a_k\},
\]
taken relative to $D_k$, with $I_k(C) = D_k$ when no
$c \in C$ targets $a_k$. Then $I_k(C)$ is an interval of
$D_k$ or empty.
\end{lemma}
\noindent
\begin{proof}
By \cref{lem:totality}, every $\sem{c}$ in the intersection
is a non-empty interval of the totally ordered set $D_k$. The
intersection of finitely many intervals of a totally ordered
set is an interval of $D_k$ or
empty~\cite{davey2002introduction}.\qed
\end{proof}
\noindent
In \cref{ex:bsb}, the offer's set $C_1 = \{c_1, c_2, c_3\}$
constrains each axis once, so $I_{\text{width}}(C_1) =
(0, 1920]$, $I_{\text{height}}(C_1) = (0, 1080]$, and
$I_{\text{depth}}(C_1) = (0, 50]$. Adding
$(\texttt{oaap:width}, \odrl{gteq}, 100)$ to $C_1$ yields the
normalised width interval $(0, 1920] \cap [100, \infty) =
[100, 1920]$.
\noindent
Two intervals on the same axis are disjoint precisely when one
lies entirely below the other, including the case where they
meet at a single value that at least one excludes.
\emph{Endpoint precedence} makes this precise.
\begin{definition}[Endpoint precedence]
\label{def:precedence}
Let one interval have upper endpoint $u$ and another have lower
endpoint $l$. We write $u \prec l$ (``$u$ precedes $l$'') when the
first interval ends before the second begins with no common value,
which holds when either $u < l$, so a gap separates them, or
$u = l$ with at least one endpoint open, so they touch at a single
value one interval excludes.
\end{definition}
\noindent
For instance, $(0, 1200)$ and $[1200, \infty)$ meet at
$u = l = 1200$ with the upper endpoint open, so $u \prec l$ and
they are disjoint. By contrast, $[0, 1200]$ and $[1200, \infty)$
both contain $1200$, so $u \not\prec l$ and they overlap there.
\noindent
\begin{theorem}[Conflict criterion]
\label{thm:criterion}
Let $c_1, c_2$ be well-formed constraints over the same axis
$a_k$, with denotations having lower endpoints $l_1, l_2$
and upper endpoints $u_1, u_2$. Then
\[
  \sem{c_1} \cap \sem{c_2} = \emptyset
  \;\iff\;
  u_1 \prec l_2 \;\vee\; u_2 \prec l_1.
\]
\end{theorem}
\begin{proof}
By \cref{lem:totality}, $\sem{c_1}$ and $\sem{c_2}$ are non-empty
intervals of the totally ordered $D_k$. Two such intervals are
disjoint iff one lies entirely below the other, equivalently iff
the upper endpoint of one precedes the lower endpoint of the other
under $\prec$ (\cref{def:precedence})~\cite{davey2002introduction},
which is the disjunction $u_1 \prec l_2 \vee u_2 \prec l_1$.\qed
\end{proof}
\noindent
For the width pair of \cref{ex:bsb}, $\sem{c_1} = (0, 1920]$ and
$\sem{c'_1} = \{2400\}$, so $u_1 = 1920 < 2400 = l_2$, hence
$u_1 \prec l_2$ and the intervals are disjoint.
\noindent
The conflict criterion decides a single axis, but a policy generally constrains several axes at once, and because
these are independent, the region satisfying all its constraints
is fixed by the per-axis intervals alone.
\begin{definition}[Box Denotation]
\label{def:box-denotation}
For a finite conjunction $C$ of well-formed axis-specific
constraints over $\Decomp{\ell}$, the \emph{box denotation}
of $C$ is
\[
  \sem{C} = I_1 \times \cdots \times I_n,
\]
where $I_k$ is the normalised interval of $C$ on axis $a_k$
(\cref{lem:normalisation}). The result is an \emph{axis-aligned box} in
$\mathcal{V}_\ell$~\cite{de2008computational}.
\end{definition}
\noindent
For \cref{ex:bsb}, $\sem{C_1} = (0, 1920] \times (0, 1080] \times (0, 50]$. For the request's $C_2$, the depth axis is unconstrained, so its normalised interval is $D_{\text{depth}} = (0, \infty)$ and
$\sem{C_2} = \{2400\} \times \{800\} \times (0, \infty)$. This full depth
interval is $C_2$'s permitted region, not a stated bound: the verdict
(\cref{def:box-verdict}) reads depth as \Unknown, since $C_2$ does not
mention it.
\noindent
Because the box is axis-aligned, deciding membership in it never
couples the axes, and a vector lies in $\sem{C}$ exactly when each coordinate lies in its own interval.
\begin{theorem}[Projection]
\label{thm:projection}
For any value vector $\mathbf{v} = (v_1, \ldots, v_n) \in
\mathcal{V}_\ell$, $\mathbf{v} \in \sem{C}$ iff $v_k \in I_k$
for every $k$.
\end{theorem}
\begin{proof}
Immediate from the Cartesian product structure of
\cref{def:box-denotation}~\cite{de2008computational}.\qed
\end{proof}
\noindent
By \cref{thm:projection}, $\sem{C} = \emptyset$ whenever two
constraints in $C$ conflict on one axis (e.g.\
$\texttt{oaap:width}~\odrl{lteq}~600$ and
$\texttt{oaap:width}~\odrl{gteq}~800$), an intra-policy
contradiction distinct from inter-policy conflict between
$C_1$ and $C_2$.

\noindent
The box denotation sends every constraint conjunction to an
axis-aligned box.
\begin{theorem}[Axis-aligned box expressibility]
\label{thm:aabb}
For every axis-aligned box $I_1 \times \cdots \times I_n$ in
$D_1 \times \cdots \times D_n$ with each $I_k$ a non-empty
interval of $D_k$, there is a conjunction $C$ of well-formed
axis-specific constraints with
$\sem{C} = I_1 \times \cdots \times I_n$.
\end{theorem}
\noindent
\begin{proof}
For each axis $a_k$, every interval shape on $D_k$ has a
syntactic counterpart in \cref{def:interval-denotation}. The
five non-trivial shapes $[v_k, v_k]$, $[\inf D_k, v_k]$,
$[v_k, \sup D_k]$, $[\inf D_k, v_k)$, $(v_k, \sup D_k]$
correspond to $\odrl{eq}, \odrl{lteq}, \odrl{gteq}, \odrl{lt},
\odrl{gt}$. The full domain $D_k$ corresponds to no constraint
on $a_k$ (\cref{def:box-denotation}), and a bounded interval
$[l_k, u_k]$ uses one lower-endpoint and one upper-endpoint
constraint. Conjoining the axis-wise counterparts across all
$n$ axes yields the required $C$.\qed
\end{proof}
\noindent
The box denotation is exact only for axis-aligned regions. A
non-axis-aligned region is over-approximated by its smallest
enclosing box, so a conflict inside the boxes is detected but
one arising only outside them is not. For example, an
aspect-ratio constraint $w / h \leq r$ on
$\texttt{oaap:absoluteSizeWidth}$ and
$\texttt{oaap:absoluteSizeHeight}$ has a smallest enclosing box
$D_w \times D_h$, which admits pairs that violate the ratio, so
conflict detection on it is sound but not complete.
\vspace{-10pt}
\section{Conflict Detection, Composition, and Runtime}
\label{sec:conflict}
Conflict detection has two layers. A per-axis verdict compares two intervals on one axis and is always \Conflict\ or \Compatible. A box-level verdict
aggregates these and adds a third value, \Unknown, when an axis is constrained
by only one policy. The three are totally ordered,
$\Conflict \sqsubset \Unknown \sqsubset \Compatible$
(\cref{def:verdict-algebra}).
\begin{definition}[Verdict Algebra]
\label{def:verdict-algebra}
The \emph{verdict set}
$\Verdicts = \{\Conflict, \Unknown, \Compatible\}$ is
totally ordered by $\Conflict \sqsubset \Unknown \sqsubset
\Compatible$. Strong Kleene conjunction $\wedge_K$ and
disjunction $\vee_K$ are $\min$ and $\max$ under $\sqsubset$
respectively. For instance, $\Conflict \wedge_K \Compatible
= \Conflict$.
\end{definition}
The Conflict Criterion (\cref{thm:criterion}) already decides
whether two intervals on an axis are disjoint. The per-axis
verdict records that outcome as a value in $\Verdicts$ so it can
be aggregated across axes, empty intersections giving $\Conflict$ and non-empty giving $\Compatible$.
\begin{definition}[Per-axis verdict]
\label{def:axis-verdict}
For axis $a_k$ and constraint sets $C_1, C_2$ with normalised
intervals $I_k(C_1)$ and $I_k(C_2)$, the \emph{per-axis
verdict} $V_k \in \Verdicts$ is
\[
  V_k =
  \begin{cases}
    \Conflict   & I_k(C_1) \cap I_k(C_2) = \emptyset,\\
    \Compatible & I_k(C_1) \cap I_k(C_2) \neq \emptyset.
  \end{cases}
\]
\end{definition}
\noindent
For \cref{ex:bsb}, the width pair gives
$(0, 1920] \cap \{2400\} = \emptyset$, so
$\axV{\text{width}} = \Conflict$. The height pair gives
$(0, 1080] \cap \{800\} = \{800\}$, so
$\axV{\text{height}} = \Compatible$. The depth axis is
constrained only by the offer, so the per-axis verdict does
not apply and is handled at the box level.

\noindent
Conflict is preserved under refinement. A constraint refined to a
subset of its denotation retains any existing conflict.
\begin{lemma}[Conflict propagation under refinement]
\label{lem:conflict-propagation}
Let $c_1, c_1', c_2$ be well-formed constraints on the same
axis with $\sem{c_1} \subseteq \sem{c_1'}$. If
$\sem{c_1'} \cap \sem{c_2} = \emptyset$, then
$\sem{c_1} \cap \sem{c_2} = \emptyset$.
\end{lemma}
\noindent
\begin{proof}
$\sem{c_1} \subseteq \sem{c_1'}$ and
$\sem{c_1'} \cap \sem{c_2} = \emptyset$ imply
$\sem{c_1} \cap \sem{c_2} = \emptyset$.\qed
\end{proof}
\noindent
For \cref{ex:bsb}, tightening BSB's width bound from
$\odrl{lteq}~1920$ to $\odrl{lteq}~1500$ refines its
denotation to $(0, 1500] \subset (0, 1920]$, and the conflict
with the request's $\odrl{eq}~2400$ persists:
$(0, 1500] \cap \{2400\} = \emptyset$.

\subsection{Box-Level Verdicts}
A per-axis verdict decides a single axis, but comparing two policies requires one verdict across all of them, which the box-level verdict obtains by aggregating the per-axis outcomes. When only one policy constrains an axis,
there is no basis for comparison: silence on that axis is not a statement that
any value is allowed, so the verdict is \Unknown\ rather than a guess.
\begin{definition}[Box Verdict]
\label{def:box-verdict}
For constraint sets $C_1, C_2$ over $\Decomp{\ell}$, the
verdict on axis $a_k$ extends the per-axis verdict
(\cref{def:axis-verdict}) with \Unknown:
\[
  V_k =
  \begin{cases}
    \Conflict   & \text{if } I_k(C_1) \cap I_k(C_2) = \emptyset,\\
    \Compatible & \text{if } I_k(C_1) \cap I_k(C_2) \neq \emptyset,\\
    \Unknown    & \text{if $C_1$ or $C_2$ does not constrain }a_k,
  \end{cases}
\]
where the first two cases apply when both $C_1$ and $C_2$
constrain $a_k$. The \emph{box verdict} is
$\boxV(C_1, C_2) = \min(V_1, \ldots, V_n)$ under $\sqsubset$.
\end{definition}

\noindent
For \cref{ex:bsb}, $V_1 = \Conflict$ (width),
$V_2 = \Compatible$ (height), and $V_3 = \Unknown$ (depth,
since $C_2$ does not constrain it), so
$\boxV(C_1, C_2) = \min\{\Conflict, \Compatible, \Unknown\}
= \Conflict$.
\noindent
At the box level, refining a policy can only move the verdict
downward under $\sqsubset$, toward \Conflict\ and never toward
\Compatible. This lifts \cref{lem:conflict-propagation} to the whole box verdict.
\begin{proposition}[Monotonicity]
\label{prop:monotone}
If $I_k(C_1') \subseteq I_k(C_1)$ for some axis $a_k$ and all
other axes are unchanged, then
$\boxV(C_1', C_2) \sqsubseteq \boxV(C_1, C_2)$.
\end{proposition}
\begin{proof}
On axis $a_k$, $I_k(C_1') \subseteq I_k(C_1)$ and
$I_k(C_1) \cap I_k(C_2) = \emptyset$ imply
$I_k(C_1') \cap I_k(C_2) = \emptyset$
(\cref{lem:conflict-propagation}), so $V_k' \sqsubseteq V_k$.
The remaining verdicts are unchanged, hence
$\min(V_1', \ldots, V_n') \sqsubseteq \min(V_1, \ldots, V_n)$
under $\sqsubset$.\qed
\end{proof}
\noindent
Conflict detection asks whether two policies overlap, a
symmetric question. Box containment asks the asymmetric one,
whether one policy's permitted region lies inside another's, as
when a request must fall within an offer. An axis left
unconstrained by either policy yields $\Unknown$ rather than
$\Compatible$, since neither policy states its intended scope on
that axis.
\begin{definition}[Box Containment]
\label{def:box-containment}

For constraint sets $C_1, C_2$ over $\Decomp{\ell}$, the per-axis
containment status $S_k$ is
\[
  S_k =
  \begin{cases}
    \Unknown    & \text{if $C_1$ or $C_2$ does not constrain }a_k,\\
    \Compatible & \text{else if } I_k(C_1) \subseteq I_k(C_2),\\
    \Conflict   & \text{otherwise,}
  \end{cases}
\]
and the \emph{box containment} $\boxS(C_1, C_2)$ is
\[
  \boxS(C_1, C_2) =
  \begin{cases}
    \Compatible & \text{if every }S_k = \Compatible,\\
    \Conflict   & \text{if some }S_k = \Conflict,\\
    \Unknown    & \text{otherwise.}
  \end{cases}
\]
\end{definition}
\noindent
For \cref{ex:bsb}, testing whether BnF's request is contained
in BSB's offer, $\boxS(C_2, C_1)$ has $S_1 = \Conflict$ on
the width axis because $\{2400\} \not\subseteq (0, 1920]$, so
the box containment is \Conflict.
\noindent
The \Unknown\ verdict is genuine: completing the missing axis can make the two policies either \Compatible\ or \Conflict, so neither is yet justified.
\cref{def:completion} formalises such extensions.
\begin{definition}[Completion]
\label{def:completion}
A \emph{completion} of $C_1, C_2$ is a pair
$(\hat{C}_1, \hat{C}_2)$ with $C_i \subseteq \hat{C}_i$ such
that for every axis $a_k$ unconstrained by $C_i$, the set
$\hat{C}_i$ adds exactly one well-formed constraint targeting
$a_k$, for example $(\ell_{a_k}, \odrl{eq}, v_k)$ with
$v_k \in D_k$.
\end{definition}
\begin{theorem}[Unknown Soundness]
\label{thm:unknown-sound}
$\boxV(C_1, C_2) = \Unknown$ iff no axis constrained by both
policies yields \Conflict\ and at least one axis is left
unconstrained. Moreover, \Unknown\ is sharp. A completion
with $\boxV = \Compatible$ exists, and, provided each
unconstrained $D_k$ admits two distinct values, a completion
with $\boxV = \Conflict$ exists.
\end{theorem}
\begin{proof}
\textit{Characterisation.} $\boxV = \min_k V_k$ equals
\Unknown\ iff no $V_k = \Conflict$ and at least one
$V_k = \Unknown$, which is the stated condition.
\textit{Compatible completion.} For each unconstrained
$a_k$, add $(\ell_{a_k}, \odrl{eq}, v_k)$ to both
$\hat{C}_1$ and $\hat{C}_2$ for some $v_k \in D_k$. Then
$I_k(\hat{C}_1) = I_k(\hat{C}_2) = \{v_k\}$, so
$V_k = \Compatible$, and axes already constrained by both
have $V_k = \Compatible$ by hypothesis. Hence
$\boxV(\hat{C}_1, \hat{C}_2) = \Compatible$.
\textit{Conflict completion.} Pick an unconstrained $a_k$
and two distinct values $v_1 < v_2$ in $D_k$. Add
$(\ell_{a_k}, \odrl{lteq}, v_1)$ to $\hat{C}_1$ and
$(\ell_{a_k}, \odrl{gteq}, v_2)$ to $\hat{C}_2$. Then
$I_k(\hat{C}_1) \cap I_k(\hat{C}_2) = \emptyset$ by
\cref{thm:criterion}, so $V_k = \Conflict$ and
$\boxV(\hat{C}_1, \hat{C}_2) = \Conflict$.\qed
\end{proof}
\noindent
For a variant of \cref{ex:bsb} where BnF requests width
$\odrl{eq}~1500$ (compatible with BSB) and does not constrain
depth, $\boxV = \Unknown$. Adding
$(\texttt{oaap:depth}, \odrl{eq}, 25)$ to BnF's request gives
a \Compatible\ completion, since $\{25\} \subseteq (0, 50]$.
Adding $(\texttt{oaap:depth}, \odrl{gteq}, 60)$ instead gives
a \Conflict\ completion, since
$[60, \infty) \cap (0, 50] = \emptyset$.
\vspace{-10pt}
\subsection{Logical Constraint Composition}
\label{sec:composition}

\odrl{and} composition is already given by
\cref{def:box-denotation}: same-axis constraints intersect,
cross-axis constraints form Cartesian products. The remaining
operators \odrl{or} and \odrl{xone} compose multiple boxes,
and their verdict is a function of the box verdicts of all
branch pairs.

\begin{definition}[Branch and branch pair]
\label{def:branch}
A \emph{branch} $B$ is a single constraint or a conjunction
of constraints within a logical constraint, treated as a
constraint set with denotation $\sem{B}$ and box verdict
given by \cref{def:box-denotation,def:box-verdict}. A
\emph{branch pair} $(B_i, B_j)$ takes one branch from each of
the two policies being compared.
\end{definition}
\noindent
A disjunction is compatible if some branch pair is
compatible, and conflicting only if every branch pair
conflicts.

\begin{definition}[Disjunction verdict]
\label{def:or-verdict}
For policies whose constraint sets are \odrl{or}\ compositions
of branches,
\[
  \verdict_{\odrl{or}} =
  \begin{cases}
    \Compatible & \text{if }\exists (B_i, B_j):
                  \boxV(B_i, B_j) = \Compatible,\\
    \Conflict   & \text{if }\forall (B_i, B_j):
                  \boxV(B_i, B_j) = \Conflict,\\
    \Unknown    & \text{otherwise.}
  \end{cases}
\]
\end{definition}
\noindent
An exclusive disjunction is compatible only if exactly one
branch pair is compatible and every other pair conflicts.

\begin{definition}[Exclusive-disjunction verdict]
\label{def:xone-verdict}
For policies whose constraint sets are \odrl{xone}\
compositions of branches,
\[
  \verdict_{\odrl{xone}} =
  \begin{cases}
    \Compatible & \text{if exactly one }(B_i, B_j)\text{ has }
                  \boxV = \Compatible\\
                & \text{and every other pair has }
                  \boxV = \Conflict,\\
    \Conflict   & \text{if }\forall (B_i, B_j):
                  \boxV(B_i, B_j) = \Conflict,\\
    \Unknown    & \text{otherwise.}
  \end{cases}
\]
\end{definition}

\noindent
For \cref{ex:bsb}, suppose BnF's request is a two-branch
\odrl{or}: a kiosk branch
$(\texttt{oaap:width}~\odrl{eq}~2400,\,
\texttt{oaap:height}~\odrl{eq}~800)$ or an archival branch
$(\texttt{oaap:width}~\odrl{eq}~1200,\,
\texttt{oaap:height}~\odrl{eq}~800,\,
\texttt{oaap:depth}~\odrl{eq}~25)$. Paired against BSB's
$C_1$, the kiosk branch conflicts on width, while the
archival branch is compatible on all three axes. Since one
branch pair is compatible, $\verdict_{\odrl{or}} =
\Compatible$.
\begin{theorem}[Composition Soundness]
\label{thm:composition-soundness}
A \Conflict\ verdict under conjunction, disjunction, or
exclusive disjunction of two policies implies their
denotations are disjoint.
\end{theorem}
\begin{proof}
\textit{Conjunction.} Some $V_k = \Conflict$ means
$I_k(C_1) \cap I_k(C_2) = \emptyset$; by \cref{thm:projection},
$\sem{C_1} \cap \sem{C_2} = \emptyset$.
\textit{Disjunction.} Every branch pair conflicts, so
\[
  \Bigl(\textstyle\bigcup_i \sem{B_i}\Bigr) \cap
  \Bigl(\textstyle\bigcup_j \sem{B_j}\Bigr) =
  \textstyle\bigcup_{i,j}\bigl(\sem{B_i} \cap \sem{B_j}\bigr)
  = \emptyset.
\]
\textit{Exclusive disjunction.} A \Conflict\ verdict for
\odrl{xone} also requires every branch pair to conflict, so
the argument is identical.\qed
\end{proof}

\begin{remark}[Soundness, not completeness]
\label{rem:incomplete}
The framework is sound by \cref{thm:composition-soundness}
but not complete: when the intended region is not
axis-aligned, a conflict outside the axis-aligned box may go undetected.
\end{remark}
\begin{theorem}[Non-interference]
\label{thm:conservative-extension}
OAAP leaves non-interval ODRL constraints unchanged: every
constraint $(\ell, \bowtie, v)$ over a non-interval $\ell$
is satisfiable under OAAP iff it is satisfiable under ODRL.
\end{theorem}
\begin{proof}
OAAP introduces only the axis operands
$\ell_{a_k} \in \Decomp{\ell}$ for interval $\ell$.
Non-interval operands are unchanged, so their satisfaction
conditions are identical under both.\qed
\end{proof}
\subsection{Runtime Evaluation}
\label{sec:runtime}
The static verdicts of \cref{sec:conflict,sec:composition}
compare two policies. At runtime, the W3C ODRL Formal
Semantics draft~\cite{odrlfs2026} specifies that an Evaluator
decides whether an Evaluation Request, an assignment of
values to left operands, satisfies a policy's constraint set, producing a control state of \emph{permit} or \emph{deny}.
\begin{definition}[Request]
\label{def:request}
A \emph{request} on $\Decomp{\ell}$ is a partial mapping
$\rho \colon \Decomp{\ell} \rightharpoonup \mathbb{R}$ with
$\rho(\ell_{a_k}) \in D_k$ whenever defined.
\end{definition}
\noindent
A request satisfies a constraint when the requested value
lies in the constraint's denotation; satisfaction lifts
compositionally to constraint sets.
\begin{definition}[Satisfaction]
\label{def:satisfaction}
$\rho$ \emph{satisfies} an atomic constraint
$c = (\ell_{a_k}, \bowtie, v_k)$, written $\rho \models c$,
iff $\rho(\ell_{a_k})$ is defined and lies in $\sem{c}$.
Satisfaction extends to logical compositions of branches
$B_1, \ldots, B_m$:
\begin{align*}
  \rho \models \odrl{and}(B_1, \ldots, B_m)
    &\iff \rho \models B_j \text{ for every } j,\\
  \rho \models \odrl{or}(B_1, \ldots, B_m)
    &\iff \rho \models B_j \text{ for some } j,\\
  \rho \models \odrl{xone}(B_1, \ldots, B_m)
    &\iff \rho \models B_j \text{ for exactly one } j.
\end{align*}
\end{definition}
\noindent
For BnF's request $\rho$ with $\rho(\texttt{oaap:width}) =
2400$ and $\rho(\texttt{oaap:height}) = 800$ in
\cref{ex:bsb}, $\rho$ does not satisfy BSB's $C_1$: since
$2400 \notin (0, 1920]$, $\rho$ fails the width constraint,
and a single failure breaks an \odrl{and} conjunction.
\begin{theorem}[Runtime Soundness]
\label{thm:runtime-sound}
If $\verdict(C_1, C_2) = \Conflict$ under any composition of
\cref{sec:composition}, then no request $\rho$ satisfies
both $C_1$ and $C_2$.
\end{theorem}
\begin{proof}
By \cref{thm:composition-soundness},
$\sem{C_1} \cap \sem{C_2} = \emptyset$. Suppose, for
contradiction, $\rho \models C_1$ and $\rho \models C_2$.
Then the value vector
$(\rho(\ell_{a_1}), \ldots, \rho(\ell_{a_n}))$ lies in both
$\sem{C_1}$ and $\sem{C_2}$ by \cref{thm:projection},
contradicting their disjointness.\qed
\end{proof}
\noindent
The W3C Evaluator's control state for a permission with
constraint set $C$ on a request $\rho$ is therefore
\emph{permit} when $\rho \models C$ and \emph{deny}
otherwise. A \Conflict\ verdict between two policies thus
guarantees that no Evaluation Request yields \emph{permit}
for both.
\section{Profile Implementation and Evaluation}
\label{sec:implementation}
We instantiate the framework of \cref{sec:problem,sec:conflict} as OAAP, a profile that introduces 15 axis-specific left operands. The profile adds no operator or value type and preserves the meaning of every operand it does not decompose, so it is a minimal extension of ODRL.
\begin{table}[t]
\footnotesize
\centering
\caption{The 15 axis-specific left operands of OAAP. Value
domain~$D_k$ is fixed by the profile.}
\label{tab:profile-operands}
\small
\begin{tabular}{@{}lll@{}}
\toprule
\textbf{Base operand} & \textbf{Axis operand (IRI)} & $D_k$ \\
\midrule
\multirow{3}{*}{\texttt{absoluteSize}}
  & \texttt{oaap:absoluteSizeWidth}  & $(0, \infty)$ \\
  & \texttt{oaap:absoluteSizeHeight} & $(0, \infty)$ \\
  & \texttt{oaap:absoluteSizeDepth}  & $(0, \infty)$ \\
\addlinespace
\multirow{3}{*}{\texttt{relativeSize}}
  & \texttt{oaap:relativeSizeWidth}  & $(0, 100]$ \\
  & \texttt{oaap:relativeSizeHeight} & $(0, 100]$ \\
  & \texttt{oaap:relativeSizeDepth}  & $(0, 100]$ \\
\addlinespace
\multirow{3}{*}{\texttt{absoluteSpatialPosition}}
  & \texttt{oaap:absoluteSpatialPositionX} & $\mathbb{R}$ \\
  & \texttt{oaap:absoluteSpatialPositionY} & $\mathbb{R}$ \\
  & \texttt{oaap:absoluteSpatialPositionZ} & $\mathbb{R}$ \\
\addlinespace
\multirow{3}{*}{\texttt{relativeSpatialPosition}}
  & \texttt{oaap:relativeSpatialPositionX} & $[0, 100]$ \\
  & \texttt{oaap:relativeSpatialPositionY} & $[0, 100]$ \\
  & \texttt{oaap:relativeSpatialPositionZ} & $[0, 100]$ \\
\addlinespace
\multirow{3}{*}{\texttt{spatialCoordinates}}
  & \texttt{oaap:spatialCoordinatesLongitude} & $[-180, 180]$ \\
  & \texttt{oaap:spatialCoordinatesLatitude}  & $[-90, 90]$ \\
  & \texttt{oaap:spatialCoordinatesAltitude}  & $\mathbb{R}$ \\
\bottomrule
\end{tabular}
\end{table}
OAAP is realised as an OWL~2~\cite{owl2-overview} ontology and a
SHACL~\cite{knublauch2017shacl} shapes graph. The ontology declares
the axis operands as \texttt{LeftOperand} instances, each linked to
its base operand via SKOS~\cite{miles2009skos}; the shapes constrain
\texttt{odrl:rightOperand} to \texttt{xsd:decimal} within the
per-axis bounds of \cref{tab:profile-operands}, rejecting malformed
policies at authoring time and fixing the domains used in automated reasoning. Because each axis is a totally ordered domain, a constraint over it always denotes an interval, and the profile interprets the
heterogeneous units and coordinate systems of its operands
uniformly, Cartesian and geographic alike. The same mechanism extends to any further ordered domain, such as time, where a constraint bounds an instant or a duration. Domains that are not totally ordered, and operators that do not denote an interval, lie outside its scope.

\vspace{-10pt}
\subsection{Evaluation}
\label{sec:eval}
Each  pairs two constraints with a conjecture asserting the
expected verdict, and is emitted in three interoperable
representations from a single source: the policy in
Turtle~(\texttt{.ttl}), a first-order conjecture in
TPTP~\cite{sutcliffe2017tptp}~(\texttt{.p}), and, where the claim reduces to a ground or existential query over linear arithmetic, an SMT-LIB encoding~\cite{BarFT-RR-25}~(\texttt{.smt2}). The Turtle
policies follow the constraint patterns of the BSB--BnF
cultural Dataspace scenario (\cref{ex:bsb}) and the domain bounds of the OAAP profile (\cref{tab:profile-operands}), expressed in valid
ODRL syntax. The profile's interval and verdict algebra axioms are
encoded once and shared across problems.

The theorems of \cref{sec:problem,sec:conflict,sec:runtime} are established by the hand proofs given there; the benchmark checks
concrete instances predicted by those results. Each static problem
encodes a constraint pair with the verdict its governing theorem
predicts, which the provers confirm or refute under the shared
axioms; for the Runtime group the provers \emph{execute} the
satisfaction relation of \cref{def:satisfaction}, since
$\rho \models c$ reduces to interval membership. We use
E~3.2.5~\cite{schulz2019eprover} and
Vampire~5.0.1~\cite{kovacs2013vampire} as first-order engines and
Z3~4.15.4~\cite{demoura2008z3} and cvc5~1.3.3~\cite{barbosa2022cvc5}
on the SMT-encoded problems. Verdicts map to the SZS status
ontology~\cite{SZS2008}: \Conflict\ to \texttt{Theorem}
(SMT~\texttt{unsat}), \Compatible\ to a satisfying witness
(\texttt{Theorem}/\texttt{Satisfiable}, SMT~\texttt{sat}), and the
satisfiability tiers to \texttt{Satisfiable}/\texttt{Unsatisfiable}.
A problem is \emph{verified} when each applicable engine returns its
declared SZS status and none returns a contradicting one.
\begin{table}[t]
\footnotesize
\centering
\caption{Benchmark groups, in the order the paper develops them.
$N$ is the problem count; every problem is verified by all four
provers (E, Vampire, Z3, cvc5).}
\label{tab:audit-suite}
\setlength{\tabcolsep}{5pt}
\renewcommand{\arraystretch}{1.15}
\begin{tabularx}{\linewidth}{@{}l c X@{}}
\toprule
Group & $N$ & Validates \\
\midrule
Interval \& box denotation & 76 &
  Per-axis intervals, 2D/3D boxes, projection, normalisation,
  well-formedness
  (\cref{def:interval-denotation,def:box-denotation,thm:projection,lem:normalisation,def:well-formed}) \\
Conflict criterion & 32 &
  Endpoint-precedence disjointness and open/closed boundary cases
  (\cref{thm:criterion,def:precedence}) \\
Logical composition & 45 &
  \odrl{and}/\odrl{or}/\odrl{xone} verdict composition
  (\cref{def:or-verdict,def:xone-verdict,thm:composition-soundness}) \\
Containment \& completion & 33 &
  Box containment and \Unknown\ completion sharpness
  (\cref{def:box-containment,thm:unknown-sound}) \\
Runtime & 10 &
  Request satisfaction, permit/deny
  (\cref{def:satisfaction,thm:runtime-sound}) \\
Soundness (negative) & 40 &
  Deliberately wrong verdicts; each refutation witnesses soundness \\
Satisfiability classes & 20 &
  Constraint-set satisfiability and unsatisfiability (SZS
  \texttt{Satisfiable}/\texttt{Unsatisfiable}) \\
\bottomrule
\end{tabularx}
\end{table}
All 256 problems are verified with full cross-prover and cross-formalism concordance: E and Vampire return the declared status on all 256 first-order problems, and Z3 and cvc5 agree on the 253 problems carrying an SMT encoding (three
are first-order only). Every counter-satisfiable conjecture fails to derive in
both formalisms, and no solver returns a contradicting status. The four
solvers discharge the suite in 10.7\,s of combined CPU time (about 21\,s
wall-clock) on an Intel Core Ultra 7 165U with 32\,GB RAM, no problem
exceeding two seconds.
\vspace{-10pt}
\section{Related Work}
\label{sec:related}
Pucella and
Weissman~\cite{pucella2006} gave the first first-order foundation for ODRL, establishing decidability of permission entailment. Steyskal and Polleres~\cite{steyskal2015} proposed a rule-based semantics for ODRL covering action dependencies; Bonatti et al.~\cite{bonatti2020machine} developed a formal semantics aligned with the W3C Evaluator model, refined recently by Bonatti, Fornara, and Harth~\cite{bonatti2025towards} into a declarative, trace-based semantics covering permissions with duties, prohibitions with remedies, and obligations. Their left-operand catalogue (e.g., size, resolution) assigns each operand a single scalar codomain, leaving the interval fragment outside the formalisation, and the W3C ODRL Community Group has since published a draft formal semantics specification~\cite{odrlfs2026} building on this line of work. Salas et al.~\cite{salas2025odrl} give a query-based evaluation and comparison semantics for ODRL, formalising asymmetric provider–requester  conflict via query containment; their subsequent work~\cite{salas2026normalisation} reduces
policy containment to set equality by splitting numeric intervals at
every right-operand value. The W3C ODRL Community Group also
maintains domain profiles. The Data Spaces~\cite{odrl-ds-profile}
and Temporal~\cite{odrl-temporal-profile} profiles extend ODRL's
vocabulary with new terms, defined informally and without a formal
semantics for the added terms. Closer to the constraint layer,
Cano-Benito et al.~\cite{canobenito2024time} align ODRL's operators
with the OWL-Time interval relations, an operator-level extension without a conflict semantics.

Two strands thus leave the multi-axis case open. The formal
semantics treat each left operand as a single scalar domain, so
interval operands are out of scope or reduced to a scalar reading
and the axis ambiguity remains; the profile and operator extensions
add terms or operators without a formal semantics for what they
introduce. OAAP addresses both at the constraint layer: it is a
profile extension whose decomposed operands carry a denotational
semantics, scalar per axis, and the resulting constraint sets admit
the query containment, interval normalisation, and denotational
verdict procedures of the formalisations above.
\vspace{-10pt}
\section{Conclusion and Future Work}
\label{sec:conclusion}
Multi-axis assets in data spaces must be constrained per axis, yet
ODRL provides no axis identity at the constraint layer.
Five spatial ODRL left operands range over multi-axis domains,
while the constraint syntax provides a single scalar slot without axis identity. Conflict detection over these operands is therefore unsound or incomplete. The OAAP profile decomposes each multi-axis operand into axis-specific scalar operands, 15 across the five
affected terms, within ODRL constraint syntax. Each
constraint names a single axis, and operands outside the
decomposition are unchanged. Each constraint denotes an interval
per axis and each policy an axis-aligned box, under a three-valued
verdict algebra over $\{\Conflict, \Compatible, \Unknown\}$. We
establish axis-aligned-box expressibility, monotonicity under
refinement, completion sharpness, and soundness from per-axis
comparison through logical composition to runtime evaluation. The method is not specific to the spatial operands: it applies
to any operand over a totally ordered domain, so an operand a
future ODRL revision may add over such a domain, e.g., a temporal
one bounding an instant or a duration, is covered by the same interval semantics. On a single-axis operand the decomposition is the identity, so the value lies instead in the interval
denotation and three-valued verdict procedure, which give sound conflict,
compatibility, and containment detection even where no dimensional ambiguity
arises. The
results are validated on 256 TPTP and SMT-LIB problems with
Vampire, E, Z3, and cvc5. The decomposed constraints are scalar
per axis, so an off-the-shelf solver discharges the same encodings
at evaluation time; Z3 decides them in milliseconds, supporting
runtime conflict checking. The existing evaluators
ODRE~\cite{cimmino2024opendigitalrightsenforcement} and
ODRL-TestSuite~\cite{slabbinck2025interoperable} accept the
constraints without modification.

\textbf{Future work.} OAAP does not yet handle non-axis-aligned regions
(e.g.\ circular or ratio-bounded shapes) or the ordered $\odrl{andSequence}$ operator; both remain open. We also plan to integrate OAAP into the Eclipse
Dataspace Components (EDC) connector of the Culture Dataspace.
\vspace{-10pt}
\section*{Declaration of use of Generative AI}
The authors used Claude (Anthropic) and Grammarly for grammar and
sentence refinement, reviewed all content, and take full
responsibility for the publication's content.

\bibliographystyle{splncs04}
\bibliography{references}

\end{document}